\documentclass[runningheads]{llncs}
\usepackage{booktabs}
 \usepackage{multirow}
\newtheorem{assumption}{Assumption}

\usepackage[T1]{fontenc}
\usepackage{graphicx}
\usepackage{xcolor}
\usepackage{pifont}
\usepackage{amsmath,amssymb}

\newcommand{\NN}{\mathbb N}
\newcommand{\pNN}{\mathbb N^+}

\newcommand{\RR}{\mathbb R}

\newcommand{\eps}{\varepsilon}

\newcommand{\todo}[1]{}

\newcommand{\Adomain}{\mathcal{A}}

\newcommand{\Gdomain}{\mathcal{G}}
\newcommand{\Pdomain}{\mathcal{P}}

\newcommand{\Sdomain}{\mathcal{S}}

\newcommand{\Wdomain}{\mathcal{W}}
\newcommand{\Xdomain}{\mathcal{X}}
\newcommand{\Ydomain}{\mathcal{Y}}
\newcommand{\Zdomain}{\mathcal{Z}}

\newcommand{\Grandomvar}{G}

\newcommand{\Prandomvar}{P}

\newcommand{\Urandomvar}{U}

\newcommand{\Wrandomvar}{W}
\newcommand{\Yrandomvar}{Y}
\newcommand{\Xrandomvar}{X}
\newcommand{\Zrandomvar}{Z}

\newcommand{\Avar}{a}
\newcommand{\Gvar}{g}

\newcommand{\Pvar}{p}

\newcommand{\Uvar}{u}

\newcommand{\Wvar}{w}
\newcommand{\Yvar}{y}
\newcommand{\Xvar}{x}
\newcommand{\Zvar}{z}

\newcommand{\run}{\rho}

\newcommand{\head}{1}
\newcommand{\tail}{0}

\newcommand{\prob}{\mathbb{P}}
\newcommand{\expe}{\mathbb{E}}

\newcommand{\indi}{\mathbf{1}}

\newcommand{\distr}{\Delta}

\newcommand{\env}{\theta}
\newcommand{\Env}{\Theta}

\newcommand{\fair}{\varphi}

\newcommand{\Mon}{\mathfrak{M}} 
\newcommand{\Enf}{\mathfrak{E}}
\newcommand{\Intervals}{\mathcal{I}}
\newcommand{\interval}{I}

\newcommand{\conf}{\delta}
\newcommand{\error}{\varepsilon}

\newcommand{\cost}{\ell}

\newcommand{\mc}{\mathcal{M}}

\newcommand{\regE}{\mathtt{R}}

\newcommand{\regC}{\mathtt{C}}

\newcommand{\hori}{h}

\newcommand{\cross}{\ding{55}}

\begin{document}
\title{Algorithmic Fairness: A Runtime Perspective}
%
%
\author{Filip Cano \and
Thomas A. Henzinger \and
Konstantin Kueffner}
\authorrunning{F. Cano, T. A. Henzinger, and K. Kueffner}
%
\institute{Institute of Science and Technology Austria \\
\email{\{filip.cano, tah, konstantin.kueffner\}@ist.ac.at}}

\maketitle              
\begin{abstract}
Fairness in AI is traditionally studied as a static property evaluated once, over a fixed dataset. However, real-world AI systems operate sequentially, with outcomes and environments evolving over time. This paper proposes a framework for analysing fairness as a runtime property. Using a minimal yet expressive model based on sequences of coin tosses with possibly evolving biases, we study the problems of monitoring and enforcing fairness expressed in either toss outcomes or coin biases. Since there is no one-size-fits-all solution for either problem, we provide a summary of monitoring and enforcement strategies, parametrised by environment dynamics, prediction horizon, and confidence thresholds. For both problems, we present general results under simple or minimal assumptions. We survey existing solutions for the monitoring problem for Markovian and additive dynamics, and existing solutions for the enforcement problem in static settings with known dynamics.

\end{abstract}

\section{Introduction}





Fairness is essential for the responsible deployment of AI in numerous domains such as hiring, credit scoring, and criminal justice~\cite{dressel2018accuracy,obermeyer2019dissecting,scheuerman2019computers,liu2018delayed,berk2021fairness}. 
When AI influences decisions affecting human lives, equitable treatment is crucial to prevent the reinforcement of societal biases. 
%
%
Traditionally, fairness has been considered a \emph{static} property of AI systems, meaning that fairness is evaluated under the assumption of having access to a fixed snapshot of the world~\cite{dwork12,hardt2016equality}.
This approach assumes that the environment, user behaviour, and data distribution are observable and remain stable over time. 
These assumptions often fail in dynamic, real-world applications where AI systems interact with people and adapt to new inputs.


Recognising fairness as a \emph{sequential} property introduces a more nuanced and realistic understanding of the interaction of AI systems with the world.
In runtime settings, input-output pairs can only be accessed sequentially,
so fairness analysis is restricted to the evidence obtained by single traces.
Furthermore, the distribution of future inputs may change, either from exogenous processes or affected by the system's own past decisions.

In this paper, we focus on sequences of biased coin tosses, where each coin may have a different, potentially evolving bias. 
This minimal yet expressive setting allows us to explore fundamental questions about group fairness at runtime, abstracting away the complexities of real-world domains. 
The aim of this paper is to serve as a roadmap for the research of fairness from a runtime perspective, unifying existing results into a coherent framework and motivating future research on the topic.

In the coin flip setting, fairness can refer to the average outcome of the coin tosses, to the average bias, or to the bias of the most recent coin.
We show how these measures can be extended from finite to infinite sequences and how they can be interpreted over stochastic processes via conditional expectations, giving rise to a notion of fairness at runtime.

We then study two core problems: (1) the monitoring problem, in which the system passively observes outcomes and must estimate fairness within a confidence interval; and (2) the enforcement problem, in which the system can actively intervene to ensure fairness remains within a target range. For both problems, we survey recent work that approaches them under certain assumptions and also provide some novel results.

In this paper we provide an overview of the different concepts of fairness at runtime, the different variants of both the monitoring and the enforcement problem, and the existing approaches to both problems under concrete assumptions.
We parametrise both problems by the environment dynamics, the fairness property under study, the prediction horizon, i.e., how many steps in the future we are interested in, and the confidence value, which indicates the probability of obtaining a correct solution.
For the enforcement problem, we add a set of target intervals where the fairness property should lie, as an extra parameter.
In both the monitoring and enforcement problems, certain assumptions about the problem's parameters lead to tailored efficient solutions, showing that there is no one-size-fits-all solution.

We connect our theoretical framework with several recent results in runtime fairness.
For the monitoring problem, we show that under no assumptions on the environment dynamics, finite outcome fairness can be easily monitored, but infinite outcome cannot.
We then show how to monitor any fairness property with a finite horizon under the assumption that there is a single (unknown) coin. 
We outline the work of Henzinger et al.~\cite{henzinger2023monitoring,henzinger2023partial} on monitoring systems whose dynamics are governed by a Markov chain, as well as their work on monitoring systems that follow additive environmental dynamics~\cite{henzinger2023dynamic}.

For the enforcement problem, we show that any fairness property can be enforced with high confidence, disregarding cost, under mild assumptions on the target intervals.
We then study the enforcement of outcome fairness in depth in the simplest setting: a single coin of known bias. We survey work from Cano et al.~\cite{fairnessshieldsAAAI2025} on synthesising enforcers with different guarantees on cost-optimality and confidence value.
We also propose a new enforcer with less restrictive assumptions than those from Cano et al. and outline an extension of their work to systems where the bias of the coin changes according to known dynamics.

Taken together, the surveyed results and the unifying framework we introduce provide a comprehensive foundation for reasoning about fairness in sequential decision-making systems at runtime.

\section{Fairness Definitions}
\subsection{Preliminaries}
We denote the set of all natural numbers as $\NN$, the set of all positive natural numbers as $\pNN$, and the set of all real numbers as $\RR$. 
Let $\Adomain$ be any set. We denote the set of all length $n\in \NN$ sequences over $\Adomain^n$. Moreover, we denote the set of all sequences up to length $n$ as $\Adomain^{\leq n}$, the set of all finite sequences as $\Adomain^*$, and the set of all infinite sequences as $\Adomain^{\omega}$.
Let $m,n\in \NN\cup\{\infty\}$ such that $m\leq n$ and let $\Uvar\in \Adomain^*\cup\Adomain^{\omega}$ be a sequence of length $n$, we denote the prefix of length $m$ as $\Uvar_{1:m}= \Avar_1, \dots, \Avar_m$, where for $m=\infty$ we have $\Uvar_{1:m}=\Uvar$.


\subsection{Fairness over Sequences}
We limit our investigation of quantitative fairness measures to a simple coin flip setting.
We consider fairness on finite and infinite sequences.
We define two fairness measures over finite sequences. 
One measure quantifies fairness based on the coin tosses, the other does so over their biases. We generalise each of them by taking the limit.

\paragraph{Setting.}
We investigate fairness in a sequence of coins and coin toss outcomes, denoted by $\Wvar = (\Pvar_t, \Xvar_t)_{t \in \NN}$.
At each step $t\in \pNN$, the value $\Pvar_t\in [0,1]$ represents the bias of the coin and the value $\Xvar_t\in \{\tail, \head\}$ represents the outcome of the coin toss, where $\tail$ indicates tail and $\head$ indicates head. We denote the set of bias-outcome pairs as $\Wdomain=\Pdomain\times \Xdomain$ where $\Pdomain\subseteq [0,1]$ is a subset set of all coins and $\Xdomain=\{\tail, \head\}$. 
Note that we will use the terms ``bias of a coin'' and ``coin'' interchangeably. 

\paragraph{Outcome fairness.}
In the outcome measure, we evaluate fairness only on the outcomes of the coin tosses. The measure compares the average number of heads and tails. It ranges from $0$ (tails-biased) to $1$ (heads-biased), with $1/2$ indicating fairness.
Formally, for a sequence 
$\Wvar\in \Wdomain^{\omega}$ 
and $t\in \pNN$ the \emph{outcome fairness} of the prefix $\Wvar_{1:t}=\Wvar_1, \dots, \Wvar_t$ is defined as 
\begin{equation}
    \label{eq:outcome}
   \fair_O(\Wvar_{1:t}) = \frac{1}{t}  \sum_{i=1}^t \Xvar_i.
\end{equation}
\paragraph{Bias fairness.}
In the bias measure, we evaluate fairness on the average biases of the coins in the sequence, instead of the toss outcomes.
As before, a value close to $1/2$ indicates a fair sequence.
Formally, for a sequence 
$\Wvar\in \Wdomain^{\omega}$ 
and $t\in \pNN$ the \emph{bias fairness} of the prefix $\Wvar_{1:t}=\Wvar_1, \dots, \Wvar_t$ is defined as 
\begin{equation}
    \label{eq:average}
   \fair_B(\Wvar_{1:t}) = \frac{1}{t}  \sum_{i=1}^t \Pvar_i.
\end{equation}
\paragraph{Current fairness.}
In the current measure, we evaluate fairness on the bias of the current coin in the sequence.
As before, a value close to $1/2$ indicates a fair coin.
Formally, for a sequence 
$\Wvar\in \Wdomain^{\omega}$ 
and $t\in \pNN$ the \emph{current fairness} of the prefix $\Wvar_{1:t}=\Wvar_1, \dots, \Wvar_t$ is defined as 
\begin{align}
   \fair_C(\Wvar_{1:t}) =  \Pvar_t.
\end{align}
\paragraph{Infinite outcome and bias.}
We lift the fairness properties to infinite traces by taking the limit, i.e., for $\fair\in \{\fair_O, \fair_B, \fair_C\}$, whenever it exists,
\begin{align}
    \fair(\Wvar)= \lim_{t\to \infty} \fair(\Wvar_{1:t}).
\end{align}

\subsection{Fairness over Stochastic Processes}
In this section, we enrich our setting by considering the sequence of coins as a stochastic process.
With this additional structure, we can study the average and outcome fairness not only on concrete sequences, but also in expectation with respect to the dynamic evolution of the stochastic process, 
and conditioning on a prefix of the process.

\paragraph{Setting.}
We assume the sequence of bias-outcome pairs $\Wvar=(\Wvar_t)_{t\in \pNN}\in \Wdomain^{\omega}$ is a realization of a stochastic process $\Wrandomvar=(\Prandomvar_t, \Xrandomvar_t)_{t\in \pNN}$ with support $\Wdomain^{\omega}$.
The stochastic process is characterised by a dynamics function $\env\colon \Wdomain^* \to \distr(\Pdomain)$ mapping the past sequence into a probability distribution over the following coin. The process is defined by the following dynamics.
At each time step $t\in \pNN$ and the history $\Wvar_1, \dots, \Wvar_{t-1} \in \Wdomain^{t-1}$ we define
\begin{align}
    \Prandomvar_t\sim \env(\Wvar_1, \dots, \Wvar_{t-1} ) \quad \text{and} \quad \Xrandomvar_t \sim \mathrm{Bernoulli}(\Pvar_t)
\end{align}
where $\Pvar_t$ is a realization of $\Prandomvar_t$.
Intuitively, the next coin's bias $p_t$ is sampled from the process dynamics $\env$ according to the history $\Wvar_{1:t-1}$, 
and the toss outcome is then drawn according to this bias.

\paragraph{Runtime fairness.}
Runtime fairness captures how the expected value of a fairness property evolves over time as the stochastic process unfolds. This requires the use of the conditional expectation.
Formally, given a fairness property $\fair\in \{\fair_O, \fair_B, \fair_C\}$, the stochastic process $\Wrandomvar$, a prediction horizon $\hori\in \NN \cup \{\infty\}$, and the realized prefix $\Wvar_{1:t}=\Wvar_1, \dots, \Wvar_t\in \Wdomain^t$ at the current point in time $t\in \pNN$, we evaluate the \emph{runtime fairness} of the stochastic process for the measure $\fair$ and prediction horizon $\hori$ as
\begin{align}
    \run_{t}^\hori(\Wrandomvar;\fair) = \expe_{\env}(\fair(\Wrandomvar_{1:t+\hori}) \mid \Wvar_{1:t}),
\end{align}


\paragraph{Static vs. dynamic.}
    Defining fairness at runtime using the conditional expectation provides more insight into the actual fairness of the process, as compared to computing the expectation of the fairness measure directly. This is because with the conditional expectation, the property value adapts to the information accumulated as the process evolves.

\section{Fairness at Runtime}

\subsection{The Monitoring Problem}

\paragraph{Monitor.}
A monitor for a fairness measure $\fair$ is a function that reads a sequence of outcomes and estimates the value of the fairness property by providing an interval where the fairness property lies with high confidence.
Formally,
a monitor is a function $\Mon\colon  \Xdomain^* \to \Intervals([0, 1])$ maps a finite sequence of outcomes into an interval over $[0,1]$. 
%


\paragraph{Problem instance.}
A problem instance consists of a set of dynamics functions $\Env$,
a fairness property $\fair\in \{\fair_O, \fair_B, \fair_C\}$, a prediction horizon $\hori\in \NN$, and
an error probability threshold $\conf\in [0,1]$. 
The set $\Env$ encodes the assumption imposed onto the dynamics function used to create the stochastic process $\Wrandomvar$. We want a monitor that provides a correct verdict irrespective of the actual dynamics function $\env$, as long as it is in $\Env$.

\paragraph{Soundness.}
Let $\Mon\colon  \Xdomain^* \to \Intervals([0, 1])$ be a monitor. 
Given a problem instance $(\Env, \fair, \hori ,\conf)$ we call the monitor
\emph{pointwise sound}, if 
\begin{align}
    \forall \env \in \Env \colon \forall t\in \pNN \colon \prob_{\env}\big( \run_{t}^\hori( \Wrandomvar ;\fair) \in \Mon(\Urandomvar_{1:t}) \big)\geq 1-\conf,
\end{align}
and \emph{uniformly sound}, if 
\begin{align}
    \forall \env \in \Env \colon \prob_{\env}\big(  \forall t\in \pNN \colon \run_{t}^\hori( \Wrandomvar ;\fair) \in \Mon(\Urandomvar_{1:t}) \big)\geq 1-\conf.
\end{align}
Intuitively, we want the monitor to output the tightest interval possible while remaining sound.
\begin{problem}
    \label{prob:monitoring}
    Given a problem instance $(\Env, \fair, \hori,\conf)$, synthesise a pointwise sound or uniformly sound monitor. 
\end{problem}

\subsection{The Enforcement Problem}
\paragraph{Enforcer.}
An enforcer (often also called \emph{shield}) for a fairness property monitors the evolution of the fairness property at runtime and can overwrite some of the coin biases or outcomes with the goal of keeping the property inside a target interval. 
Formally, an enforcer is a function $\Enf:\Wdomain^* \to \distr(\Wdomain)$.
Intuitively, we want an enforcer that intervenes as little as possible.

\paragraph{Enforced process.}
The presence of an enforcer modifies the distribution of bias-outcome pairs as follows.
For each time step $t\in \pNN$ and bias-outcome sequence $\Wvar_1, \dots, \Wvar_{t-1} \in \Wdomain^{t-1}$ we define $ \Prandomvar_t'\sim \env(\Wvar_1, \dots, \Wvar_{t-1} ) $ and $ \Xrandomvar_t' \sim \mathrm{Bernoulli}(\Pvar_t')$. 
The actual bias-observation pair $(\Pvar_t,\Xvar_t)$ is a realization of the distribution defined by the enforcer w.r.t.\ the history and the realized bias-observation pair of the process $(\Pvar_t', \Xvar_t')$, i.e., 
\begin{align}
    (\Prandomvar_t, \Xrandomvar_t) \sim \Enf(\Wvar_1, \dots, \Wvar_{t-1}, (\Pvar_t', \Xvar_t')).
\end{align}
The resulting stochastic process is now characterised by the dynamics function $\env$ and the enforcer $\Enf$.

\paragraph{Problem instance.}
A problem instance consists of a set of dynamics functions $\Env$,
a fairness property $\fair\in \{\fair_O, \fair_B, \fair_C\}$, 
a prediction horizon $\hori\in \NN$, and an error probability threshold $\conf\in (0,1)$, 
and a sequence of intervals $\interval =(\interval_t)_{t\in\NN} \subseteq \Intervals([0,1])$, where the fairness property should lie in at each step.

\paragraph{Soundness.}
Let $\Enf\colon  \Wdomain^* \to \distr(\Wdomain)$ be an enforcer. 
Given a problem instance $(\Env, \fair, \hori,\conf, \interval)$ we call the enforcer
\emph{pointwise sound}, if 
\begin{align}
    \forall \env \in \Env \colon \forall t\in \pNN \colon \prob_{(\env, \Enf)}\big( \run_{t}^\hori( \Wrandomvar ;\fair) \in I_{t}) \big)\geq 1-\conf,
\end{align}
and \emph{uniformly sound}, if 
\begin{align}
    \forall \env \in \Env \colon \prob_{(\env, \Enf)}\big(  \forall t\in \pNN \colon \run_{t}^\hori( \Wrandomvar ;\fair) \in I_t \big)\geq 1-\conf.
\end{align}
Intuitively, we want an enforcer that intervenes as little as possible.

\begin{problem}
\label{prob:enforcement}
    Given a problem instance $(\Env, \fair, \hori,\conf, \interval)$, synthesise a pointwise sound or uniformly sound enforcer. 
\end{problem}


\subsection{Soundness and Quality of Monitors and Enforcers}
\paragraph{Pointwise vs. uniform soundness.}
We distinguish between pointwise and uniform soundness.
A pointwise sound monitor or enforcer has a small error probability at every point in time. 
This implies that it almost surely will make a mistake, i.e., the computed or target interval will fail to capture the actual fairness measure at least once on an infinite run.
By contrast, a uniformly sound monitor or enforcer guarantees that the invariant: ``the actual fairness measure is contained within the computed or target interval at all times'', holds with high probability. 
Note that for both monitors and enforcers, uniform soundness implies pointwise soundness, but the converse does not hold.

\paragraph{The quality of a monitor.}
We have vaguely stated that the tighter the computed intervals, the better.
In practice, we can compare one sound monitor with another for a given bias-outcome sequence by comparing the lengths of the computed intervals.
However, in order to compare two monitors on a problem instance, we would need to define how to aggregate these comparisons on finite sequences to the whole problem. This is out of the scope of this paper.

\paragraph{The quality of an enforcer.}
In analogy with monitors, we have informally argued that effective enforcers should minimise interventions. 
We chose to leave the notion of optimality deliberately vague, as little research has been done in formalizing it.
Moreover, not all interventions carry the same weight. Intuitively, altering an outcome that strongly aligns with the coin’s inherent bias constitutes a more significant intervention than one made under near-fair conditions. Likewise, modifying the bias slightly is less intrusive than drastically shifting it toward the opposite extreme.
In some of the cases we study, we formalise this intuition with an associated cost function $\cost\colon \Wdomain^*\times \Wdomain\to \RR_{\geq 0}$.
The cost function assigns a history-dependent cost to changing the current bias-outcome pair, and should satisfy that for all $\Wvar\in\Wdomain^*$ and all $(\Pvar,\Xvar)\in\Wdomain$, the cost of non-intervention is null, i.e., 
$
\cost((w,(\Pvar,\Xvar)),(\Pvar,\Xvar)) = 0.
$

\paragraph{Enforcement of bias or outcome.}
With our definition of enforcer as a function $\Enf\colon \Wdomain^* \to \distr(\Wdomain)$, we allow enforcement over both bias and outcome. 
However, enforcing both bias and outcome at the same time may not make sense.
We can think of enforcement as changing the bias of the coin before it is tossed or as changing the outcome of the coin toss.
In the first case, we are overwriting the bias $p_t$ of the coin by another value $p_t'$, so it only makes sense that the outcome follows a Bernoulli of parameter $p_t'$.
In the second case, we are tossing the coin with bias $p_t$, and overwriting the outcome, so this outcome is no longer correlated with the bias of any coin.
In this paper, we only consider enforcers that overwrite bias our outcome, but not both.


\subsection{Warm-up Examples}

\subsubsection{Outcome Fairness Monitor.}
\label{subsubsec:finite outcome}
We construct a pointwise sound or uniformly sound monitor for the set of all dynamics functions $\Env$ and the outcome fairness measure with prediction horizon of $0$.

\begin{assumption}[Outcome fairness with null horizon]
    \label{ass:warm1}
    We assume the problem to be $(\Env, \fair_O, 0,\conf)$, where $\Env$ is the set of all dynamics functions.
\end{assumption}

\paragraph{Monitor construction.}
We construct a monitor solving Problem~\ref{prob:monitoring} under Assumption~\ref{ass:warm1} using a single register $\regE$. 
This monitor incrementally updates the register to store the current value of $\fair_O$. The register is initialised with the value $0$. At time $t\in \pNN$ the value of this register is $\regE_{t-1}$. The monitor observes a new outcome $\Xvar_t$ and updates the register value as follows:
\begin{align*}
    \regE_t \gets (\Xvar + (t-1)\cdot \regE_{t-1})/ t.
\end{align*}
Because the outcomes are observed, the register value at time $t$ is guaranteed to be equal to the outcome fairness measure. 
This constitutes a general solution for the problem of monitoring outcome fairness with a null prediction horizon, so we will not discuss it further.

\subsubsection{Limit Outcome Fairness Monitor.}
It is impossible to construct a pointwise sound or uniformly sound monitor for the set of all dynamics functions $\Env$ and the outcome fairness measure with infinite prediction horizon.

\begin{assumption}[Outcome fairness with infinite horizon]
    \label{ass:warm2}
    We assume the problem to be $(\Env, \fair_O, \infty,\conf)$, where $\Env$ is the set of all dynamics functions.
\end{assumption}

\paragraph{Counterexample.}
Because $\Env$ contains all possible dynamics functions, it also contains the following set of dynamics functions 
$\Env_A=\{\env_k^A \mid k\in \NN \cup \{\infty\}\}$, defined as follows. 
For each $k\in \NN  \cup \{\infty\}$, the dynamics function $\env^A$ chooses the coin with bias $0$ for all $t\leq k$, otherwise it chooses the coin with bias $1$. 
In any concrete realisation $\Wvar\in \Wdomain^{\omega}$ generated by $\env_{\infty}^A$, it is impossible for the monitor to know whether the limit of the sequence is $0$ or $1$.

\subsubsection{Process-agnostic Fairness Enforcement.}
\label{sec:process-agnostic-enforcement}

Depending on the fairness property and the target intervals, the enforcement problem may be unfeasible. 
Current fairness can always be enforced by choosing at each time $t$ a bias $p_t$ inside the target interval.
However, outcome and bias fairness are built as averages over the sequence, so they cannot shift too fast.
\begin{lemma}
\label{lem:difference_fairness_step}
    Let $w = (p_t,x_t)_{t\in\NN}$ be a sequence of coins and outcomes, 
    and let $\fair\in\{\fair_O, \fair_B\}$.
    For all $t\geq 2$ we have $|\fair(w_{1:t}) - \fair(w_{1:(t-1)})| \leq 1/t$.
\end{lemma}
A reasonable restriction is to assume that the intersection of the target intervals is non-empty, i.e., there is a value of the fairness property that is considered ``fair'' at all times.
\begin{assumption}[Enforceable bias fairness]
\label{assumption:interval_intersection}
We assume the enforcement problem to be $(\Env, \fair_B, \hori, \conf, \interval)$, where $\cap_{t\in\NN}I_t \neq \emptyset$.
\end{assumption}
This is enough to enforce bias fairness, as we can find 
$p_\cap \in \cap_{t} I_t$ and build the enforcer
$\Enf\colon \Wdomain^*\to\Wdomain$ for all $w\in\Wdomain^*$ as $\Enf(w) = (p_\cap,0)$.
\begin{theorem}
\label{thm:blind_bias_enforcement}
    The previous enforcer solves Problem~\ref{prob:enforcement} under Assumption~\ref{assumption:interval_intersection}.
\end{theorem}
For outcome fairness, we need a stronger condition.
%
\begin{assumption}[Enforceable outcome fairness]
\label{assumption:interval_intersection_padded}
    We assume the enforcement problem to be $(\Env, \fair_O, \hori, \conf, \interval)$, where $\interval=(\interval_t)_{t\in\NN}$ is such that there exists $p\in [0,1]$ satisfying for all $t\in\NN$ the condition
    $[\max(0,p-1/t), \min(1, p+1/t)]\subseteq I_t$.
\end{assumption}
With this condition, we can build the following enforcer for outcome fairness:
\begin{equation}
\label{eq:bruteforce_outcome_enforcement}
    \Enf(w_{1:t}) = \begin{cases}
        1 & \mbox{ if } \fair_O(w_{1:t}) \leq p,\\
        0 & \mbox{ otherwise. }
    \end{cases}
\end{equation}
\begin{theorem}
    \label{thm:blind_outcome_enforcement}
    The enforcer in Eq.~\eqref{eq:bruteforce_outcome_enforcement} solves Problem~\ref{prob:enforcement} under Assumption~\ref{assumption:interval_intersection_padded}.
\end{theorem}
While these enforcers always exist, they bear no consideration on minimising interventions, and Assumptions~\ref{assumption:interval_intersection} and~\ref{assumption:interval_intersection_padded} are so weak that reasoning about that is challenging.

\section{Fairness Monitoring with Unknown Dynamics}
There does not exist a monitor capable of solving Problem~\ref{prob:monitoring} in full generality. The difficulty of the problem arises from the structure of the set of dynamics functions and the choice of the fairness measure.

\subsection{Static Coins}
\label{subsubsec:static}
If the dynamics are restricted to the set of fixed coins, we can construct a pointwise sound or uniformly sound monitor for all fairness measures and all prediction horizons. We define $\env_{\Pvar}$ to be the constant dynamics function, mapping every history into the point measure on $\Pvar\in [0,1]$, i.e., into a coin with bias $\Pvar\in [0,1]$.
This problem was first studied in~\cite{albarghouthi2019fairness}.
We present their solution to pointwise monitoring and propose an improved solution for uniformly sound monitoring.

\begin{assumption}[Constant dynamics]
    \label{assumption:all_coins}
     We assume the problem to be $(\Env, \fair, \hori,\conf)$, where $\Env$ is the set of all constant dynamics function, $\fair\in \{\fair_O,\fair_B,\fair_C\}$, and $\hori\in \NN\cup\{\infty\}$.
\end{assumption}

\paragraph{Monitor construction.}
For this class of processes, all properties can be monitored. As in Section~\ref{subsubsec:finite outcome}, the monitor maintains the register $\regE$,
which provides a point estimate of the fairness value. An error bound is then added to this estimate to form the monitoring interval. 
The computation of the error depends on whether a pointwise or uniformly sound monitor is required; that is, the error is computed differently in each case. 
They are, respectively,
\begin{align*}
    \error_t^p &= \sqrt{\frac{\log\left(2/\delta\right)}{2t} } \quad\text{and}\quad 
    \error_t^u = \sqrt{\frac{1.1    \left(2 \log\left(\pi\log(t) / \sqrt{6}\right) + \log(2/\conf)\right)}{t}}.
\end{align*}
The output of the monitor at time $t$ is the interval $[\regE_t-\error_t, \regE_t+\error_t]$ for $\error_t\in \{\error_t^p,\error_t^u\}$.
The soundness of the intervals is a direct consequence of known concentration inequalities~\cite{hoeffding1963,howard2021time}.
It generalises to other properties because in this setting, all properties, except outcome fairness, coincide. 
Moreover, this monitor can trivially be extended to outcome fairness with horizon $\hori\in \NN$ by taking the current value of the register $\regE_t$, which equals outcome fairness with horizon $0$, and extrapolating the values of the interval to the given horizon, i.e., 
\begin{align*}
    [(t\cdot\regE_t + \hori \cdot (\regE_t-\error_t))/(t+\hori), (t\cdot\regE_t + \hori \cdot (\regE_t+\error_t))/(t+\hori)].
\end{align*}

\begin{theorem}
\label{thrm:static}
    The monitor described above solves Prob.~\ref {prob:monitoring} under Ass.~\ref{assumption:all_coins}.
\end{theorem}


\subsection{Observed Markovian Dynamics}
\label{subsec:cav}

In this section, we discuss the result presented in Henzinger et al.~\cite{henzinger2023monitoring} investigating fairness on observable Markov chains.
To match our setting to the assumption made in~\cite{henzinger2023monitoring}, 
we assume $\Pdomain=\{\Pvar^{(1)}, \dots, \Pvar^{(n)}\}$ to be finite and we restrict our attention to the set of all dynamics functions with type $\env\colon \Wdomain\to \distr(\Pdomain)$, such that the induced Markov chain over $\Wdomain$ is irreducible, i.e., every coin is visited infinitely often.

\begin{assumption}[Observed Markov chain]
    \label{assumption:mc}
    We assume the problem to be $(\Env, \fair, \hori,\conf)$, where $\fair\in \{\fair_O,\fair_B,\fair_C\}$ and $\hori\in \NN$.
    We assume $\Env$ is the set of all dynamics functions that induce a finite, irreducible Markov chain. Additionally, we assume that the monitor observes the labels of the coins, i.e., if the bias at time $t\in \NN$ is $\Pvar_t=\Pvar^{(k)}\in  \Pdomain$, the monitor observes $k\in[n]$.
\end{assumption}

\paragraph{Properties.}
The monitor is designed for a specification language consisting of arithmetic expressions over single-step transition probabilities in a Markov chain.
\begin{example}
    \label{ex:cav}
    Let $\Pdomain=\{\Pvar^{(1)},\Pvar^{(2)} \}$, the monitor can estimate the expected bias of the next coin, i.e., current fairness with horizon $1$ conditioned on a fixed current bias-outcome pair. 
    For every for every $\Wvar=(\Pvar, \Xvar)\in \Wdomain$ this property corresponds to the expression
    $
    \psi = \Pvar^{(1)} \env(\Pvar, \Xvar)(\Pvar^{(1)})+ \Pvar^{(2)} \env(\Pvar, \Xvar)(\Pvar^{(1)}).
    $
    To obtain the current fairness with horizon $1$ for the state observed at time $t$, we can deploy one monitor for each $\Wvar\in \Wdomain$ and select the appropriate verdict at every time step. 
\end{example}
As in the above example, we can construct the appropriate expressions for both the current and the biased fairness with finite prediction horizon.
Because the monitor is a general purpose monitor, designed for a much richer class of time invariant properties, it is not optimised for the specific and adapting fairness properties considered here. Hence, while sound, it is not an efficient approach. 

\paragraph{Monitor construction.}
For a given expression, e.g., the one in the example above, the monitor aggregates the observed sequence of coin labels and outcomes into a sequence of independent random variables with the same expected value as the expression. This conversion is done memory-less, utilising only a few counters. 
The monitor estimates the expected value of this sequence and constructs the pointwise and uniform error bounds, similar to the monitor for a static coin sketched in Section~\ref{subsubsec:static}.


\begin{theorem}[\cite{henzinger2023monitoring}]
    \label{thrm:mc}
    The monitor described above solves Prob.~\ref{prob:monitoring} under Ass.~\ref{assumption:mc}.
\end{theorem}




\subsection{Hidden Markovian Dynamics}
\label{subsec:hmm}
In this section we discuss the result presented in Henzinger et al.~\cite{henzinger2023partial}. In their paper, they investigate the fairness of hidden Markov chains.
In Section~\ref{subsec:cav} we assume that the monitor observes, in addition to the outcome, the label of the current coin. 
This assumption makes the setting fully observable. We now drop this assumption, i.e., the induced Markov chain is partially-observed.
To compensate, we assume that the induced Markov chain is irreducible and aperiodic with its stationary distribution as its initial distribution. Moreover, we assume knowledge of a bound on its mixing time, i.e., the time required to converge to the stationary distribution~\cite{henzinger2023partial}.

\begin{assumption}[Hidden Markov chain]
    \label{assumption:hmm}
     We assume the problem to be $(\Env, \fair, \hori,\conf)$, where $\fair\in \{\fair_O,\fair_B,\fair_C\}$ and $\hori=\infty$.
    We assume $\Env$ is the set of all dynamics functions that induce an irreducible, aperiodic, finite Markov chain with a mixing time bounded by $\tau_{mix}$, and starting in its stationary distribution $\eta\in \distr(\Wdomain)$.
\end{assumption}

\paragraph{Properties.}
The monitor is designed for a specification language consisting of arithmetic expressions over expected values $\expe_{\eta}(f(\Xvar_1, \dots, \Xvar_n))$ of a given bounded function with arity $n\in \NN$ evaluated over outcomes, i.e., $f\colon \Xdomain^n\to [a,b]$. The expectation is taken w.r.t. the stationary distribution $\eta\in \distr(\Wdomain)$ of a partially observed Markov chain. Intuitively, the stationary distribution equals the proportion of time spent in each state in $\Wdomain$. We give a small example below.
\begin{example}
    \label{ex:rv}
    Let $\Pdomain=\{\Pvar^{(A)},\Pvar^{(B)} \}$, the monitor can estimate the limit average bias of the process $\Wrandomvar$, i.e., bias fairness with horizon $\infty$, by estimating
    \begin{align*}
    \psi = \expe_{\eta}(\indi[\Xrandomvar_1=1])
    \end{align*}
   The above expressions equals the proportion of heads observed over an infinite run, i.e., the limit of bias fairness.
\end{example}
For stationary Markov chains the properties $\fair_O$, $\fair_B$, and $\fair_C$ have the same value as $\psi$, if the prediction horizon is $\infty$. Unfortunately, this monitor is not suited for finite horizon properties.

\paragraph{Monitor construction.}
The monitor maintains a single register $\regE$ to incrementally estimate the value of the given function $f\colon \Xdomain^n\to [a,b]$ at runtime. 
To evaluate the function, the monitor must maintain a fixed memory of $n$. 
The register is initialised with $0$. 
After the monitor observes the first $n$ elements, it updates the register at every time $t$ as follows, 
\begin{align*}
    \regE_{t} \gets (\regE_{t-1}\cdot (t-n) + f(\Xvar_t-n, \dots, \Xvar_t))/ (t-n+1).
\end{align*}
As in Section~\ref{subsubsec:static}, it constructs an interval around $\regE$ using error bounds that depend on the mixing time.
The error bound is 
\begin{equation*}
   \sqrt{\frac{9t n^2 (b-a)^2 \tau_{mix}}{2 (t-(n-1))^2}\cdot K},
\end{equation*}
where $K = \log(2/\delta)$ for the pointwise sound monitor, 
and $K=\log(\pi^2t^2/3\delta)$ for the uniformly sound monitor.
%
The bound has to account for the mixing time of the Markov chain to avoid premature verdicts, as in the example below.
\begin{example}[Ex.~\ref{ex:rv} cont.]
    \label{ex:rv2}
    Let $\Pvar^{(A)} = 0.9$ and $\Pvar^{(B)} = 0.1$. Assume the dynamics function is defined for every $\Xvar\in \Xdomain$ and $g\in \{A, B\}$ such that $\Prandomvar \sim \env(\Pvar^{(g)}, \Xvar)$  is $\Pvar^{(g)}$ with a probability $1-\varepsilon$ for a small $\varepsilon>0$. Although we remain a long time with one coin, we can expect an equal number of heads and tails in the limit. However, if we were to deploy the monitor described in Section~\ref{subsubsec:static}, we would unduly, yet confidently, declare the limit to be biased.
\end{example}

\begin{theorem}[\cite{henzinger2023partial}]
    \label{thrm:hmm}
    The monitor described above solves Prob.~\ref{prob:monitoring} under Ass.~\ref{assumption:hmm}.
\end{theorem}

\subsection{Additive Dynamics}
\label{subsec:facct}
In this section, we discuss the result presented in Henzinger et al.~\cite{henzinger2023dynamic} on systems with additive dynamics.
We match our setting to these assumptions by restricting our attention to dynamics functions where the bias changes additively as a function of the past outcomes. 
Given an additive change function $\beta\colon \Xdomain \to \RR$ mapping a history of outcomes into a numeric value, we define the dynamics function $\env$ for every $\Wvar\in \Wdomain^{\omega}$ and every $t\in \pNN$ such that
\begin{align}
    \label{eq:lin_dyn}
   \Pvar_{t+1} = \env(\Wvar_{1:t}) =  \Pvar_t + \beta(\Xvar_t)\quad \text{and} \quad \Pvar_{1}\in\Pdomain. 
\end{align}
\begin{remark}
    Note that in the original paper, the ``bias'', i.e., the parameter subject to change, cannot exit its bounds. For simplicity of exposition, we assume this to be true as well. Moreover, the change functions can trivially be extended to the entire sequence of outcomes, i.e., $\beta\colon \Xdomain^* \to \RR$.
Moreover, an extension to fully linear dynamics is also possible. 
\end{remark}

\begin{assumption}[Additive dynamics]
    \label{assumption:linear}
    We assume the problem to be $(\Env, \fair, \hori,\conf)$, where $\fair\in \{\fair_B,\fair_C\}$ and $\hori=0$.
    We assume $\Env$ is the set of all additive dynamics functions defined in Eq.~\eqref{eq:lin_dyn}, and that the monitor has access to the change function $\beta$, but is unaware of $\Pvar_1$.
\end{assumption}

\paragraph{Properties.}
The monitor is designed to track the time conditional expectation $\expe_{t-1}(f(Z_t))$ of a bounded monotonic function $f:\Zdomain\to [a,b]$ evaluated over a random variable with additive shifting conditional expectation, i.e, $\expe_{t-1}(Z_t)$ shifts as in Eq.~\eqref{eq:lin_dyn}. 
In our setting, this equates to monitoring the current fairness with predictive horizon $0$.

\paragraph{Monitor construction.}
The monitor maintains two registers, one for the accumulated change $\regC$ and one for the bias estimate $\regE$ of the first coin, both are initialised with $0$.
At every point in time $t\in \pNN$, the monitor observes the next outcome $\Xvar_t$ and updates the registers as follows
\begin{align*}
    \regE_t\gets (\regE_{t-1}\cdot (t-1) + (\Xvar_t - \regC_{t-1})) \quad \text{and} \quad \regC_t \gets  \regC_{t-1} + \beta(\Xvar_t).
\end{align*}
The output of the monitor is an interval constructed around the current bias, i.e.,  $\regE_t+ \regC_t$, using error bounds similar to the ones in Section~\ref{subsubsec:static}. 
Because the monitor estimates the initial value, we can easily extend this result to cover bias fairness with horizon $0$.


\begin{theorem}[\cite{henzinger2023dynamic}]
    \label{thrm:dynamic}
    The monitor described above solves Prob.~\ref{prob:monitoring} under Ass.~\ref{assumption:linear}.
\end{theorem}

\section{Fairness Enforcement with Known Dynamics}
Knowing the dynamics function of the stochastic process amounts to the set $\Env$ being a singleton, i.e., $\Env = \{\env\}$.
We formalise it with the following assumption.
\begin{assumption}[Known dynamics]
    \label{assumption:known_dynamics}
    We assume the problem to be $(\Env, \fair_O, \hori, \conf, \interval)$, 
    where $\Env =\{\env\}$ is a singleton with a known dynamics function $\env$.
\end{assumption}
This problem has been studied for static systems (Assumption~\ref{assumption:all_coins}), where it reduces to a single coin of known bias. 
Although the monitoring problem is trivial, the enforcement problem is not, especially if we want to reason about cost-optimality.
To reason about cost, we assume that we are given a cost function $\cost\colon\Wdomain^*\to\Wdomain$ and that we enforce over finite or periodic time windows.



\paragraph{Finite and periodic time windows.}
We say that an enforcement problem is of finite window $T$ when the only non-trivial enforcement interval is $I_T$.
Similarly, we say that an enforcement problem is of periodic window $T$ when the only non-trivial enforcement intervals are the same and occur every $T$ time steps.
%
\begin{assumption}[Finite time window]
\label{assumption:enf:finite}
    We assume the enforcement problem to be 
    $(\Env, \fair, 0, \conf,  \interval)$, 
    where $\interval_T\subsetneq [0,1]$, and $\interval_t =[0,1]$ for all $t\neq T$.
\end{assumption}
\begin{assumption}[Periodic time window]
\label{assumption:enf:periodic}
    We assume the enforcement problem to be 
    $(\Env, \fair, 0, \conf, \interval)$, 
    where $\interval_{n\cdot T} = \interval_T \subsetneq [0,1]$, and $\interval_t =[0,1]$ for all $t\neq nT$ and all $n\in\NN$.
\end{assumption}
Since the enforcer acts only up to time $T$ in finite time window settings, the associated cost is necessarily finite.
A sound enforcer is cost-optimal if no other sound enforcer yields a lower expected cost at the beginning of the time window.

In this section we present different approaches to the enforcement problem in finite and periodic time windows under Assumptions~\ref{assumption:all_coins} and ~\ref{assumption:known_dynamics} and, at the end of the section, we sketch how the presented methods can be extended to a dynamic system, i.e., dropping Assumption~\ref{assumption:all_coins}.

\subsection{Finite Time Window}

\subsubsection{Probabilistic Guarantees.}
\label{sec:finite-prob-guarantee}
We propose an enforcement method for finite time windows with probabilistic guarantees by intervening only when the probability of the unenforced process reaching the target fairness interval falls below the confidence value.

\paragraph{Enforcer construction.}
The enforcer checks for each sequence the probability that the unenforced process reaches the target interval at time $T$, and intervenes only when it falls below $1-\conf$.
Formally, 
for $ \Wvar_{1:t} = (\Pvar_1, \Xvar_1), \dots, (\Pvar_t, \Xvar_t)  \in \Wdomain^{t}$:
\begin{equation}
\label{eq:best-effort-enforcer-general-delta}
    \Enf(\Wvar_{1:t}) = 
    \begin{cases}
        (\Pvar_t, \Xvar_t) & \mbox{ if } \prob_{\env_p}\big( \fair_O(\Wrandomvar_{1:T}) \in I_T \mid \Wvar_{1:t}\big)\geq 1-\conf, \\
        (\Pvar_t, 1-\Xvar_t) & \mbox{ otherwise.}
    \end{cases}
\end{equation}
\begin{theorem}
\label{thm:greedy_enforcers}
    The enforcer described in Eq.~\eqref{eq:best-effort-enforcer-general-delta} solves Problem~\ref{prob:enforcement} under Assumptions~\ref{assumption:all_coins},~\ref{assumption:known_dynamics}, and~\ref{assumption:enf:finite}.
\end{theorem}

\subsubsection{Deterministic Guarantees.}
\label{sec:finite-det-guarantee}
When restricting to problem instances with almost sure guarantees,
Cano et al.~\cite{fairnessshieldsAAAI2025} describe a solution to the enforcement problem that produces optimal-cost enforcers.
We encode deterministic guarantees with the following assumption.
\begin{assumption}[Almost-sure enforcement]
\label{assumption:stationary_known_det}
    We assume the problem to be $(\Env, \fair, \hori, \conf, \interval)$ with $\conf=0$.
\end{assumption}

\paragraph{Enforcer construction.}
The enforcer pre-computes, for each outcome sequence $\Xvar_{1:t}$, 
an auxiliary value function $v\colon \Xdomain^{\leq T} \to \RR$ encoding the expected cost of the optimal enforcer that guarantees that outcome fairness sits in the target interval $I_T$ after time $T$, conditioned on the first $t$ outcomes being $x_{1:t}$.
When computing $v(\Xvar_{1:t})$, we assign an infinite cost to those traces where no enforcer can guarantee fairness.
The enforcer then chooses to flip the last outcome of a sequence whenever the resulting outcome yields a lower cost.
Formally, for a given $\Wvar_{1:t} = (\Pvar_1, \Xvar_1), \dots, (\Pvar_t, \Xvar_t)\in \Wdomain^t$:
\begin{equation}
\label{eq:enforcer-shields-finite}
    \Enf(\Wvar_{1:t}) = 
    \begin{cases}
        (\Pvar_t, \Xvar_t) & \mbox{ if }
        v(\Xvar_{1:t}) \leq v(\Xvar_{1:t-1}, 1-\Xvar_t)\\
        (\Pvar_t, 1-\Xvar_t) & \mbox{ otherwise.}
    \end{cases}
\end{equation}

\begin{theorem}[\cite{fairnessshieldsAAAI2025}]
    The enforcer described in Eq.~\eqref{eq:enforcer-shields-finite} solves Problem~\ref{prob:enforcement} cost-optimally under Assumptions~\ref{assumption:all_coins},~\ref{assumption:known_dynamics},~\ref{assumption:enf:finite}, and~\ref{assumption:stationary_known_det}.
\end{theorem}
While not explored in~\cite{fairnessshieldsAAAI2025}, one could, in principle, precompute a value function $v_\delta$ representing the expected minimal cost of enforcing outcome fairness with probability $1-\delta$, and use it to synthesize sound, cost-optimal enforcers for $\delta \in (0,1)$, i.e., dropping Assumption~\ref{assumption:stationary_known_det}.
However, the recursive method proposed in~\cite[Sec. 3]{fairnessshieldsAAAI2025} for computing the value function $v$ does not suffice for computing $v_\delta$.
A naïve extension would render the synthesis procedure exponential in the time window.
Thus, a direct generalisation of the presented enforcer to arbitrary $\delta \in (0,1)$ is computationally infeasible, and different techniques will be needed to obtain optimal enforcers at general confidence levels.

\subsection{Periodic Time Window}
\label{sec:periodic-enf}
Cano et al.~\cite{fairnessshieldsAAAI2025} also propose extensions of their finite window enforcers to periodic window enforcers by repeatedly reusing or recomputing the shield every $T$ steps.
By extending the use of enforcers computed as sound and cost-optimal for a finite window, we can synthesise sound enforcers for the periodic window.

\paragraph{Enforcer construction.}
To synthesise this enforcer, we build an auxiliary value function $v'\colon \Xdomain^*\to \RR$, which extends the value function used for finite window shields.
We define the value function on time windows of size $T$, i.e., we restart the expected cost each time the outcome sequence reaches a length $n\cdot T$, for $n\in\NN$.
Let $k=n\cdot T$, we define the value function for an outcome sequence $\Xvar_{1:(k+t)}$, with $t\leq T$, as the minimal cost of enforcement from an outcome trace $\Xvar_{1:k}$ until step $k+T$, 
conditioned on the prefix $\Xvar_{1:k}$.
To emphasize this conditional dependence, we denote it as $v(\Xvar_{(k+1):(k+t)}\mid \Xvar_{1:k})$.
As in the finite window case, the enforcer compares $v(\Xvar_{(k+1):(k+t)}\mid \Xvar_{1:k})$ with $v((\Xvar_{(k+1):(k+t-1)}, 1-\Xvar_{k+t}\mid \Xvar_{1:k})$ to decide whether to interfere or not.
Formally, for a given sequence $w_{1:(k+t)}=(\Pvar_1, \Xvar_1), \dots, (\Pvar_{k+t}, _{k+t})\in \Wdomain^{k+t}$ with $1\leq t \leq T$:

\begin{equation}
\label{eq:enforcer-shields-periodic-dyn}
    \Enf(w_{1:k+t}) = 
    \begin{cases}
        (\Pvar_{k+t}, \Xvar_{k+t}) & \mbox{ if }
        v(\Xvar_{1:(k+t)}\mid \Xvar_{1:k}) \leq v\big((\Xvar_{1:(k+t-1)}, 1-\Xvar_t)\mid \Xvar_{1:k}\big)\\
        (\Pvar_{k+t}, 1-\Xvar_{k+t}) & \mbox{ otherwise.}
    \end{cases}
\end{equation}
This expected cost cannot be computed offline for all $n\in\NN$, 
but it can be computed at runtime every $T$ steps.


\begin{theorem}[\cite{fairnessshieldsAAAI2025}]
    \label{thm:dynshields}
    The enforcer described in Eq.~\eqref{eq:enforcer-shields-periodic-dyn} solves Problem~\ref{prob:enforcement} under Assumptions~\ref{assumption:all_coins},~\ref{assumption:known_dynamics},~\ref{assumption:enf:periodic}, and~\ref{assumption:stationary_known_det}.
\end{theorem}
Theorem~\ref{thm:dynshields} is a direct consequence of~\cite[Thm. 5]{fairnessshieldsAAAI2025}.
Note that the original result is stated in terms of more general group fairness properties, and as a result, the produced enforcers are not sound for $\conf=0$, but they are pointwise sound for a confidence value $\conf>0$ dependent on the input distribution.

\subsection{Enforcement on Dynamic Systems}
\label{sec:enf_dynamic}
In dynamic systems with known dynamics, the set of possible dynamics is a singleton $\Env = \{\env\}$, where $\env$ can be any dynamics function.
As far as we know, the enforcement problem has not been studied in this setting.
The approach presented in~\cite{fairnessshieldsAAAI2025} can be extended to synthesise sound enforcers for both finite and periodic windows for outcome fairness and confidence $\delta=0$, with cost-optimal guarantees for the finite window case.
To do so, we need to extend the auxiliary value function to capture the evolution of the biases. 
It has to be, therefore, a function $v\colon \Wdomain^*\to \RR$, 
encoding the expected minimal cost of enforcement for any sequence $w\in \Wdomain^*$, accounting for the dynamics of the stochastic process.
Such value functions can be efficiently computed for any dynamics function $\env$ that can be encoded as a counter automaton~\cite{fscd24}.
This extension works both for finite and periodic time windows.

\section{Related Work}

Fairness in machine learning generally involves classification or regression problems where there is a feature of the input that needs to be protected, in the sense that the model should not discriminate decisions based on the value of this feature. Typical examples of protected features include age, race, and gender~\cite{Barocas2018FairnessAM}.
While our paper is written in terms of a simple coin toss setting, it captures the main complexities arising from the study of fairness in ML settings.
Apart from the work we survey in this paper, fairness has been extensively studied from the perspective of machine learning~\cite{mehrabi2021survey,dwork2012fairness,hardt2016equality,kusner2017counterfactual,kearns2018preventing,sharifi2019average,bellamy2019ai,wexler2019if,bird2020fairlearn,zemel2013learning,jagielski2019differentially,konstantinov2022fairness}.
%
From the standard ML point of view, fairness is a static property, and fairness enforcement is undertaken at design time, be it in debiasing the training data (pre-processing methods), adding fairness-inducing regularisation terms to the loss function (in-processing) or adjusting the outputs of the model after training (post-processing). 
Recently, formal methods-inspired techniques have been used to guarantee algorithmic fairness through the verification of a learned model  \cite{albarghouthi2017fairsquare,bastani2019probabilistic,sun2021probabilistic,ghosh2020justicia,meyer2021certifying}, and enforcement of robustness \cite{john2020verifying,balunovic2021fair,ghosh2021algorithmic}.
All of these works verify or enforce algorithmic fairness statically on all runs of the system with high probability.
A critical drawback in all this body of work is that fairness is often not a static property~\cite{d2020fairness,liu2018delayed,heidari2019long}, and often fairness is satisfied too far away in the future to be practical~\cite{pmlr-v235-alamdari24a}, which justifies the need for a runtime perspective on it.
In contrast, the monitoring and enforcement techniques covered in this paper are implemented at runtime on a system that is already deployed, allowing for the flexibility to adapt fairness verdicts or enforcement on the go, as the system execution progresses.
%
Most monitors described in this paper are designed to check statistical properties, which is beyond the limit of what automata-based monitors for temporal properties can accomplish~\cite{stoller2011runtime,junges2021runtime,faymonville2017real,maler2004monitoring,donze2010robust,bartocci2018specification,baier2003ctmc}.

\section{Discussion}

\subsection{Beyond Coins: Group Fairness Properties}
In this paper, we restricted our attention to a simple coin setting. Here, we describe how this setting can be naturally generalized to analyze any group fairness property.
\todo{Group fairness properties are in general about havind different groups and treating them similarly.}

\todo{Somewhere in this section state that all the work we are surveying actually works with 2 group}

\paragraph{Example: Loan repayment.}
A classical example in the fairness literature is a loan repayment setting~\cite{d2020fairness}. 
In this setting, a member of the population arrives at a bank and applies for a loan. The bank employs a classifier to identify whether the potential customer will repay the loan. Based on the output, the bank decides to grant the loan or not. In the classical machine learning setting, it is assumed that regardless of the decision the ground truth, i.e., whether the customer repays the loan, is revealed thereafter.
Formally, a potential customer at time $t\in \NN$ is represented as a feature vector $(\Gvar_t, \Yvar_t, \Zvar_t)$, consisting of a sensitive attribute $\Gvar_t\in\Gdomain=\{A,B\}$ that indicates group membership, a hidden ground truth $\Yvar_t\in \Ydomain=\{0,1\}$ that indicates whether the customer will repay the loan ($1$) or not ($0$), and some other features $\Zvar_t\in \Zdomain$. The population at time $t$ is represented as a distribution $\mathcal{D}_t\in \distr(\Gdomain \times \Ydomain \times \Zdomain)$ over the feature space. The bank's classifier at time $t$ is a function $f_t\colon \Gdomain \times \Zdomain \to  \{0,1\}$, mapping all observable features to a decision, where $1$ indicates acceptance.

\paragraph{Commoly studied group fairness properties.}
A popular branch of fairness properties, are group fairness properties such as demographic parity and equal opportunity \cite{mehrabi2021survey}. Given the above setting, assume that we are at time $t\in \NN$.
Demographic parity compares the classifiers acceptance probabilities between the groups, i.e., 
\begin{align*}
    \prob_{\mathcal{D}_t}(f_t(\Zrandomvar_t,\Grandomvar_t)=1\mid \Grandomvar_t=A) - \prob_{\mathcal{D}_t}(f_t(\Zrandomvar_t,\Grandomvar_t)=1\mid \Grandomvar_t=B).
\end{align*}
Equal opportunity compares the classifiers acceptance probabilities between those members of the groups poised to repay the loan, i.e., 
\begin{align*}
    \prob_{\mathcal{D}_t}(f_t(\Zrandomvar_t,\Grandomvar_t)=1\mid \Grandomvar_t=A, \Yrandomvar_t= 1) - \prob_{\mathcal{D}_t}(f_t(\Zrandomvar_t,\Grandomvar_t)=1\mid \Grandomvar_t=B, \Yrandomvar_t= 1).
\end{align*}
Both properties define fairness at time $t\in \NN$ as the difference between two acceptance probabilities. Therefore, they can be expressed as the difference between the biases of two coins.
Moreover, each acceptance probability is subject to change. This change may be due to the fact that the classifier is retrained on the previously collected data or due to the fact that the distribution is influenced by the classifiers decision over time. 
Our coin setting captures those dynamics in spirit, i.e., the bias of the subsequent coin is chosen in a history dependent manner, it is limited only by the choice of underlying process. To capture the setting in full generality, the bias of the coin should be a function of the entire process, including the customer, the group, the classifiers decision, and other variables.

\paragraph{Monitoring and enforcement.}
The monitoring problem generalizes in a straightforward manner. The monitor observes the sequence of customers and classifier decisions, i.e., $(\Grandomvar_t, \Yrandomvar_t, \Zrandomvar_t, f_t(\Zrandomvar_t,\Grandomvar_t))_{t\in \NN}$, with the objective of estimating the group specific acceptance probabilities.
The enforcement problem allows for more variability. 
In the simplest setting, the enforcer can flip the classifier's decision, this is reflective of a more aggressive form of enforcement. A more nuanced approach would be to slightly nudge the classifiers decision towards being fair. For example, the enforcer could flip the classifiers decision with some small probability or lower the acceptance threshold for a specific group, e.g., if $\Zvar_t$ represents a credit score and the classifier makes its decision based on whether $\Zvar_t$ exceeds a threshold $\tau$, the enforcer could introduce group specific acceptance thresholds $\tau_A$ and $\tau_B$ to balance the acceptance rates.

\subsection{Summary}

\paragraph{Monitoring.}
Monitoring algorithmic fairness has been explored for static distributions~\cite{albarghouthi2019fairness}, observable~\cite{henzinger2023monitoring} and hidden Markov chains~\cite{henzinger2023partial,journalstatic}, and for additive and linear dynamical systems~\cite{henzinger2023dynamic,journaldynamic}.
\paragraph{Enforcement.}
The enforcement problem has received less attention and often assumes stronger restrictions.
We examined enforcement under very mild assumptions and thoroughly covered enforcement in static systems with known dynamics, where the main technical challenge is constructing outcome fairness enforcers that are cost-optimal or nearly so~\cite{fairnessshieldsAAAI2025}.

\paragraph{Overview.}
Tables~\ref{tab:monitor_summary} and~\ref{tab:enforcement_summary} summarise the monitoring and enforcement results discussed in this paper.
Each super-column represents a fairness property, i.e., outcome fairness $\fair_O$, bias fairness $\fair_B$, and current fairness $\fair_C$. 
Each sub-column represents the predictive horizon, i.e., $h\in \{0,n,\infty\}$, for the runtime interpretation of the respective fairness property. Each row represents a particular assumption on the system. The cells indicate the section where we addressed the induced monitoring or shielding problem. We also indicate which cases are trivial to solve and which are unfeasible (\cross); 
entries are left blank when a detailed treatment was not included.

\begin{table}[t]
    \centering
        \small
     \setlength{\tabcolsep}{4pt}
         \caption{Summary of all surveyed monitoring results.}
    \label{tab:monitor_summary}
    \begin{tabular}{l|ccc|ccc|ccc}
        \toprule
       \multirow{2}{*}{$\rho_t^\hori(\cdot; \fair)$} & \multicolumn{3}{c|}{$\fair_O$} & \multicolumn{3}{c|}{$\fair_B$} & \multicolumn{3}{c}{$\fair_C$} \\
        \cmidrule(r){2-10}
        & $0$ & $n$ & $\infty$ & $0$ & $n$ & $\infty$ &  $0$ & $n$ & $\infty$ \\

               \midrule    
      no assumptions (\ref{ass:warm1}\&\ref{ass:warm2}) & S.~\ref{subsubsec:finite outcome}  &   & \cross &   &   & \cross &   &   & \cross \\     known static~(\ref{assumption:all_coins}\&\ref{assumption:known_dynamics}) & \multicolumn{3}{c|}{Trivial}  & \multicolumn{3}{c|}{Trivial} & \multicolumn{3}{c}{Trivial} \\
       unknown static~\eqref{assumption:all_coins} & \multicolumn{3}{c|}{---   S.~\ref{subsubsec:static}   --- }  & \multicolumn{3}{c|}{--- S.~\ref{subsubsec:static} ---} & \multicolumn{3}{c}{---  S.~\ref{subsubsec:static}  ---} \\

    observed Markov~\eqref{assumption:mc} &  S.~\ref{subsubsec:finite outcome} & S.~\ref{subsec:cav} &  & \multicolumn{2}{c}{--- S.~\ref{subsec:cav} ---} &  & \multicolumn{2}{c}{--- S.~\ref{subsec:cav} ---} &  \\
    
    hidden Markov~\eqref{assumption:hmm} & S.~\ref{subsubsec:finite outcome}  &   & S.~\ref{subsec:hmm} &  &   & S.~\ref{subsec:hmm} &   &  & S.~\ref{subsec:hmm} \\

     additive dynamics~\eqref{assumption:linear} & S.~\ref{subsubsec:finite outcome}  &   &  &  S.~\ref{subsec:facct} & &   & S.~\ref{subsec:facct} &   &   \\

        \bottomrule
    \end{tabular}
\end{table}

\begin{table}[t]
    \centering
        \small
\caption{Summary of all surveyed enforcement results.}
    \label{tab:enforcement_summary}
\begin{tabular}{lll|ccc|ccc|ccc}
\toprule
\multicolumn{3}{l|}{\multirow{2}{*}{$\rho_t^\hori(\cdot; \fair)$}}                                                                                                                                          & \multicolumn{3}{c|}{$\fair_O$}       & \multicolumn{3}{c|}{$\fair_B$} & \multicolumn{3}{c}{$\fair_C$} \\ \cmidrule(r){4-12} 
\multicolumn{3}{l|}{}  & $\:\: \:\:  0  \:\: \:\: $ & $ \:\:  \:\: n \:\:  \:\: $ & $ \:\:  \:\: \infty \:\:  \:\: $ & $ \:\:  \:\: 0 \:\:  \:\: $ & $ \:\:  \:\: n \:\:  \:\: $ & $ \:\:  \:\: \infty \:\:  \:\: $ &  $ \:\: 0 \:\: $ & $ \:\: n \:\: $ & $ \:\: \infty \:\: $     \\ \midrule
\multicolumn{3}{l|}{no assumptions (\ref{ass:warm1}\&\ref{ass:warm2})}     & $\:$\cross$\:$ & \cross & \cross & \cross      & \cross     & \cross       & \multicolumn{3}{c}{S.~\ref{sec:process-agnostic-enforcement}}      \\
\multicolumn{3}{l|}{feasible intervals (\ref{assumption:interval_intersection}\&\ref{assumption:interval_intersection_padded})}                    & \multicolumn{3}{c|}{--- S.~\ref{sec:process-agnostic-enforcement} ---} & \multicolumn{3}{c|}{--- S.~\ref{sec:process-agnostic-enforcement} ---}       & \multicolumn{3}{c}{S.~\ref{sec:process-agnostic-enforcement}}       \\ \midrule
\multicolumn{1}{l|}{\multirow{4}{*}{\begin{tabular}[l|]{@{}l@{}} known \\ dynamics \\ (\ref{assumption:known_dynamics})\end{tabular}}} & \multicolumn{1}{l|}{\multirow{3}{*}{\begin{tabular}[l]{@{}l@{}} static \\ system (\ref{assumption:all_coins}) \end{tabular}}} & finite (\ref{assumption:enf:finite})     & S.~\ref{sec:finite-prob-guarantee} &   &   & \multicolumn{3}{c|}{Trivial}       &  \multicolumn{3}{c}{Trivial}        \\ 
\multicolumn{1}{c|}{}  & \multicolumn{1}{l|}{}                                                                               & finite a.s. (\ref{assumption:enf:finite}\&\ref{assumption:stationary_known_det})   & S.~\ref{sec:finite-det-guarantee} &   &   &  \multicolumn{3}{c|}{Trivial}        & \multicolumn{3}{c}{Trivial}        \\
\multicolumn{1}{c|}{}                                                                          &                                                                 \multicolumn{1}{l|}{}        & periodic (\ref{assumption:enf:periodic}) & S.~\ref{sec:periodic-enf} &   &   &  \multicolumn{3}{c|}{Trivial}        &  \multicolumn{3}{c}{Trivial}        \\ \cmidrule(r){2-12}
\multicolumn{1}{c|}{}                                                                          & \multicolumn{2}{l|}{Dynamic}                                                       & S.~\ref{sec:enf_dynamic} &   &   &        &       &         &       &       &        \\
\bottomrule
\end{tabular}
\end{table}
%

%

\section{Conclusion}
We have presented an overview of the work on fairness at runtime. We approach this topic through the lens of a simplified coin setting. 
In this setting, we propose unified problem statements for both the fairness monitoring problem and the fairness enforcement problem.
We restate the assumptions and results of each surveyed paper and provide the intuition of how each algorithm can be used to solve the proposed problems.

\paragraph{Future work.}
We discussed fairness properties such as current and bias fairness. 
Those two properties are the two extremes of a spectrum, where current fairness considers only the last value, and bias fairness computes the average over the entire sequence. A natural extension would be to introduce and study forms of discounting to unify both under a single property.
Moreover, a clear research direction is to combine monitoring and shielding, thus allowing the enforcement of systems with unknown dynamics.

\begin{credits}
\subsubsection*{\ackname} This work is supported by the European Research Council under Grant No.: ERC-2020-AdG 101020093.

\subsubsection*{\discintname}
The authors have no competing interests to declare that are
relevant to the content of this article. 
\end{credits}
%
%
%
\bibliographystyle{splncs04}
\bibliography{references}

\appendix

\section{Proofs}
\label{sec:proofs_appendix}

\subsection{Process-agnostic enforcement (Proof of Lemma~\ref{lem:difference_fairness_step} and Theorems~\ref{thm:blind_bias_enforcement} and~\ref{thm:blind_outcome_enforcement})}

\begin{lemma*}[\ref{lem:difference_fairness_step}]
    Let $w = (p_t,x_t)_{t\in\NN}$ be a sequence of coins and outcomes, and let $\fair\in\{\fair_O, \fair_B\}$.
    For all $t\in\geq 2$ we have 
    $\left|\fair(w_{1:t}) - \fair\left(w_{1:(t-1)}\right)\right| \leq 1/t$.
\end{lemma*}
\begin{proof}
    We prove it for $\fair = \fair_O$, the proof for $\fair_B$ follows the same argument.
    Let $\xi = \sum_{i=1}^{t-1} x_i$. Then
    \begin{equation}
        \left|\fair(w_{1:t}) - \fair\left(w_{1:(t-1)}\right)\right| = 
        \frac{\xi+x_t}{t} - \frac{\xi}{t-1} = \frac{x_t\cdot t - (x_t+\xi)}{t(t-1)}.
    \end{equation}
    If $x_t = 0$, since $\xi \leq t-1$, we have 
    \begin{equation}
        \left|\fair(w_{1:t}) - \fair\left(w_{1:(t-1)}\right)\right| = \frac{\xi}{t(t-1)} \leq 1/t.
    \end{equation}
    If $x_t = 1$, since $t-(\xi+1)\leq t-1$, we have
    \begin{equation}
        \left|\fair(w_{1:t}) - \fair\left(w_{1:(t-1)}\right)\right| = \frac{t-(\xi+1)}{t(t-1)} \leq 1/t.
    \end{equation}
    \hfill $\qed$
\end{proof}

\begin{theorem*}[\ref{thm:blind_bias_enforcement}]
    Under Assumption~\ref{assumption:interval_intersection},
    the enforcer defined for all
    $w\in\Wdomain^*$ as $\Enf(\Wvar) = (p_\cap,0)$,
    where $p_\cap\in \cap_{t} I_t$ solves Problem~\ref{prob:enforcement}.
\end{theorem*}
\begin{proof}
    Since the enforcer keeps the bias constant to $p_\cap$, bias fairness is trivially $\fair_B(\Wvar_{1:t}) = p_\cap$ for any sequence $\Wvar_{1:t}$ produced by the enforcer.
    \hfill $\qed$
\end{proof}

\begin{theorem*}[\ref{thm:blind_outcome_enforcement}]
    The enforcer in Eq.~\eqref{eq:bruteforce_outcome_enforcement} solves Problem~\ref{prob:enforcement} under Assumption~\ref{assumption:interval_intersection_padded}.
\end{theorem*}
\begin{proof}
    We prove this result by induction on $t$. 
    \paragraph{Base case.} For $t=1$, the target interval is $I_1 = [0,1]$.
    \paragraph{Inductive case.} Consider an enforced sequence $\Xvar_{1:(t-1)}$. 
    By the induction hypothesis $\fair_O(\Xvar_{1:(t-1)}) \in [\max(0,p-1/(t-1)), \min(1,p+1/(t-1))]$.
    Because of Lemma~\ref{lem:difference_fairness_step}, the only way to obtain $\fair_O(\Xvar_{1:t})$ outside of the target interval is with $\fair_O(\Xvar_{1:(t-1)}) \leq p$ and $x_t=0$, or with $\fair_O(\Xvar_{1:(t-1)}) \leq p$ and $x_t=0$. 
    We consider this two cases separately.
    \begin{itemize}
        \item If  $\fair_O(\Xvar_{1:(t-1)}) \leq p$ and $x_t=0$, then $\fair_O(\Xvar_{1:t}) \leq p$, so $\Enf$ enforces $x_t=1$. 
        The resulting sequence satisfies 
        \[
        \fair_O(\Xvar_{1:t})  \leq \fair_O(\Xvar_{1:(t-1)}, 1) +1/t \leq p+1/t.
        \]
        \item Analogously, if $\fair_O(\Xvar_{1:(t-1)}) > p$ and $x_t=1$, then  $\fair_O(\Xvar_{1:t}) > p$, so $\Enf$ enforces $x_t=0$.
        The resulting sequence satisfies
        \[
            \fair_O(\Xvar_{1:t})  \geq \fair_O(\Xvar_{1:(t-1)}, 1) -1/t > p+1/t.
        \]
    \end{itemize}
    \hfill $\qed$
\end{proof}

\subsection{Static coin monitoring (proof of Theorem~\ref{thrm:static})}

\begin{theorem*}[\ref{thrm:static}]
    The monitor described in Section~\ref{subsubsec:static} solves Problem~\ref{prob:monitoring} under Assumption~\ref{assumption:all_coins}.
\end{theorem*}
\begin{proof}
    Let $\Wvar_{1:t} =\Wvar_1, \dots, \Wvar_t$ be the trace observed up to time $t$.

   \emph{All but outcome fairness.}
   We know that $[\regE_t-\error_t, \regE_t+\error_t]$ forms a confidence interval for $\error_t=\error_t^p$ and a confidence sequence $\error_t=\error_t^u$ for $\expe(\Xrandomvar_1)=\Pvar$. This is a direct consequence of Hoeffding's inequality~\cite{hoeffding1963} and the result of Howard et al.~\cite{howard2021time}.
Because $\Pvar_t$ is unchanging and not subject to randomness we know that taking the expected $\Prandomvar_t$ will almost surely be $\Pvar$ irrespective on what we condition on. Hence, this implies that for all $t\in\pNN$ and all $\hori\in \NN$ 
\begin{align*}
   \Pvar = \expe(\Prandomvar_{t+\hori} \mid \Wvar_{1:t}) = \expe(\frac{1}{n+t}\sum_{i=1}^{t+\hori} \Pvar_i \mid \Wvar_{1:t})  .
\end{align*}
Moreover, by the law of large numbers we know that $\lim_{t\to\infty} \frac{1}{t}\sum_{i=1}^t\Xrandomvar_i=\expe(\Xrandomvar_1)=\Pvar$ almost surely, thus
\begin{align*}
   \Pvar =  \expe( \lim_{k\to\infty} \frac{1}{k}\sum_{i=1}^k\Xrandomvar_i \mid \Wvar_{1:t})  = 
   \expe(  \lim_{k\to\infty} \frac{1}{k}\sum_{i=1}^k\Prandomvar_i \mid \Wvar_{1:t})  = 
      \expe( \lim_{k\to\infty} \Prandomvar_k \mid \Wvar_{1:t}) .
\end{align*}

   \emph{Outcome Fairness.}
    For $\fair=\fair_O$ the register value $\regE_t$ equals the value of outcome fairness at runtime, i.e.,  
   \begin{align*}
      \regE_t = \frac{1}{t}\sum_{i=1}^t \Xvar_i= \expe\left(\frac{1}{t}\sum_{i=1}^t \Xvar_i \; \middle|\;  \Wvar_{1:t}\right) 
   \end{align*}
   Moreover, a direct consequence of the conditional expectation is that for all $t\in\pNN$ and all $\hori\in \NN$  
   \begin{align*}
       \expe(  \frac{1}{t}\sum_{i=1}^{\hori+t} \Xrandomvar_i  \mid \Wvar_{1:t} ) = \sum_{i=1}^t \Xvar_i + \sum_{i=t+1}^{t+\hori} \Pvar_i =   \sum_{i=1}^t \Xvar_i + \hori \Pvar
   \end{align*}
   which is estimated by 
   \begin{align*}
    [(t\cdot\regE_t + \hori \cdot (\regE_t-\error_t))/(t+\hori), (t\cdot\regE_t + \hori \cdot (\regE_t+\error_t))/(t+\hori)],
\end{align*}
following from the soundness of the error bound and the unchanging bias.
\hfill $\qed$
\end{proof}

\subsection{Observable Markov chain monitoring (proof of Theorem~\ref{thrm:mc})}

\label{appx:mc}
For a bias $\Pvar$ and outcome $\Xvar$, we denote $\eta(\Pvar,\Xvar)=\Pvar^{\Xvar} (1-\Pvar)^{1-\Xvar}$.
Let $\Pdomain=\{\Pvar^{(1)}, \dots,\Pvar^{(n)} \}$.
We can encode the stochastic process $\Wrandomvar$ generated by a dynamics function $\env$ as a Markov chain $\mc=(\Sdomain, M, \lambda)$ consisting of the state space $\Sdomain= [n]\cup[n]\times \{\tail,\head\}$, the Markov transition kernel  $M\colon \Sdomain \to \distr(\Sdomain)$, and the initial distribution $\lambda\in \distr(\Sdomain)$. The initial distribution is defined for every state $s \in \Sdomain$ as 
\begin{align*}
    \lambda(s)=
    \begin{cases}
         \env(\epsilon)(\Pvar^{(s)}) &\quad \text{if } s\in [n]  \\
         0 &\quad \text{otw.}
    \end{cases}
\end{align*}
and $0$ everywhere else.
The transition function $M\colon \Sdomain \to \distr(\Sdomain)$ is defined for every state $s\in \Sdomain$ and every successor state $s\in \Wdomain$ as
\begin{align*}
    M(s, s')=
    \begin{cases}
       \eta(\Pvar^{(k)}, \Xvar)  &\quad \text{if } s=(k)\in [n] \text{ and } s'=(k, \Xvar)\in [n]\times \{\tail,\head\} \\
        \env(\Pvar^{(k)}, \Xvar)(\Pvar^{(k')})  &\quad \text{if } s=(k,\Xvar)\in [n]\times \{\tail,\head\} \text{ and } s'=k'\in [n] \\
        0   &\quad \text{otw.} 
    \end{cases}
\end{align*}
Given a realization $\Wvar\in \Wdomain^{\omega}$ of $\Wrandomvar$, the corresponding realization of the Markov chain $(s_t)_{t\in \NN}$ is such that for every $t\in \pNN$, we have $s_{2t-1}= \Pvar_{t}$ and $s_{2t}=  \Pvar_{t}, \Xvar_t$. 
 The monitors developed in Henzinger et al.~\cite{henzinger2023monitoring} monitor time invariant expressions over Markov chains starting conditioned on some state. We need to encode our dynamic fairness properties in their formalism.

 \begin{lemma}
    \label{lemma:mc:current}
     Current fairness with time horizon $\hori\in \NN$ can be encoded as an arithmetic expression over transition probabilities of $\mc$ 
 \end{lemma}
\begin{proof}
    Assume we are given $\Wvar_1=(\Pvar^{(k)}, \Xvar)\in \Wdomain$, then we can encode the current fairness at a horizon of $\hori\in \NN$ starting in $\Wvar_1$ as an expression over transition probabilities of the Markov chain $\mc$ constructed above. 
    This expression is
    \begin{align*}
       \psi_{\hori}(\Wvar_1) &=\sum_{s_2, \dots, s_{2\hori}} M(\Wvar, s_2)\prod_{i=1}^{2\hori-1} M(s_i, s_{i+1}) \cdot M(s_{2\hori}, (s_{2\hori},1)) .
    \end{align*}
    where $\psi_0(\Wvar_1)= M(\Pvar^{(k)}, \Wvar_1)$.
    It is easy to see that this equals $\expe(\Prandomvar_{1+\hori}\mid \Wrandomvar_1=\Wvar)$, i.e., 
    \begin{align*}
        &\expe(\Prandomvar_{1+\hori}\mid \Wrandomvar_1=\Wvar) = \sum_{\Wvar_{2}, \dots, \Wvar_{\hori+1}} \Pvar_{\hori+1} \prob(\Wvar_{2}, \dots, \Wvar_{\hori+1} \mid \Wrandomvar_1=\Wvar)
        \\
        &= \sum_{\Wvar_{2}, \dots, \Wvar_{\hori+1}}
        \env(\Wvar)(\Pvar_2) \cdot \prod_{i=2}^{\hori}\left(\eta(\Pvar_{i}, \Xvar_{i}) \env(\Wvar_i)(\Pvar_{i+1})\right)\cdot \eta(\Pvar_{\hori+1},1) \\
     &= \sum_{\Wvar_{2}, \dots, \Wvar_{\hori+1}}
        M(\Wvar,\Pvar_2) \cdot \prod_{i=2}^{\hori}\left(M(\Pvar_{i}, \Wvar_{i})\right)\cdot  \prod_{i=2}^{\hori}\left( M(\Wvar_i,\Pvar_{i+1})\right)\cdot M(\Pvar_{\hori+1}, (\Pvar_{\hori+1},1)) \\
        &= \sum_{s_2, \dots, s_{2\hori}}
        M(\Wvar,s_2) \cdot \prod_{\substack{i=2 \\ i\ \text{even}}}^{2\hori}\left(M(s_{i}, s_{i+1})\right)\cdot  \prod_{\substack{i=2 \\ i\ \text{odd}}}^{2\hori}\left( M(s_i,s_{i+1})\right)\cdot M(s_{2\hori+1}, (s_{2\hori+1},1)) \\
        &=\sum_{s_2, \dots, s_{2\hori}} M(\Wvar, s_2)\prod_{i=1}^{2\hori-1} M(s_i, s_{i+1}) \cdot M(s_{2\hori}, (s_{2\hori},1)) .
    \end{align*}
    \hfill $\qed$
\end{proof}

\begin{lemma}
    \label{lemma:bias}
     Expected bias fairness with time horizon $\hori\in \NN$, without the history, can be encoded as an arithmetic expression over transition probabilities of $\mc$ 
 \end{lemma}
\begin{proof}
  As a direct consequence of Lemma~\ref{lemma:mc:current} and the linearity of expectation we know that bias fairness with horizon $\hori$ starting in $\Wvar_1$ at time $t\in \pNN$ equals
    \begin{align*}
    \chi_{\hori}(\Wvar_1)  =  \frac{1}{t+\hori}\sum_{i=1}^{t+\hori}  \psi_{i}(\Wvar_1) = \expe\left(\frac{1}{t+\hori}\sum_{i=1}^{t+\hori} \Pvar_i\;  \middle|\;  \Wvar_{1}\right).
    \end{align*}
    \hfill $\qed$
\end{proof}

\begin{theorem*}[\ref{thrm:mc}]
    The monitor described above solves Problem~\ref{prob:monitoring} under Assumption~\ref{assumption:mc}.
\end{theorem*}
\begin{proof}
    \emph{Current fairness.}
    We can monitor current fairness by deploying a monitor $\Mon_{(\Pvar^{(k)}, \Xvar)}^\hori$ for each bias-outcome pair $(\Pvar^{(k)}, \Xvar) \in \Wdomain$ and horizon $\hori$. Each monitor is synthesized for the time invariant expression $\psi_{\hori}((\Pvar^{(k)}, \Xvar))$. 
    During runtime we take a union bound over all monitors and use the fact that we can observe the label of the current coin to output the interval computed by the corresponding monitor. The soundness follows from 
    Lemma~\ref{lemma:mc:current} and the results of Henzinger et al.~\cite{henzinger2023monitoring}.

    \vspace{0.5em}
    \emph{Bias fairness.}
    Bias fairness with horizon $\hori\in \NN$ at time $t\in \pNN$ consists of two components, the realized and the expected part, i.e., 
    \begin{align*}
        \expe\left(\frac{1}{t+\hori}\sum_{i=1}^{t+\hori} \Pvar_i\;  \middle|\;  \Wvar_{1:t}\right) = \frac{1}{t+\hori}\sum_{i=1}^{t} \Pvar_i +  \expe\left(\frac{1}{t+\hori}\sum_{i=1}^{t+\hori} \Pvar_i\;  \middle|\;  \Wvar_{t}\right)
    \end{align*}
    From Lemma~\ref{lemma:bias} we know that we can monitor the expected part. We follow the same strategy as for current fairness. We deploy a monitor $\Mon_{(\Pvar^{(k)}, \Xvar)}^{(\chi,\hori)}$ for each bias-outcome pair $(\Pvar^{(k)}, \Xvar) \in \Wdomain$ and the expression $ \chi_{\hori}((\Pvar^{(k)}, \Xvar)) $, apply a union bound, and select the correct interval based on the current observed pair. 
    For the realized part we deploy a monitor $\Mon_{\Pvar^{(k)}}$ for the transition probability $M(\Pvar^{(k)}, (\Pvar^{(k)}, 1))$, which equals $\Pvar^{(k)}$. We perform a union bound to ensure all intervals hold.
    This allows us to compute the bounds
    \begin{align*}
        \left[\sum_{k\in [n]} C_t^{s}\cdot l_{t}^k,\sum_{k\in [n]} C_t^{s}\cdot u_{t}^k \right]
    \end{align*}
    where $l_{t}^k$ and $u_{t}^k$ are the lower and upper bounds computed by $\Mon_{\Pvar^{(k)}}$ at time $t$ and $C_t^{k}$ is the number of visit to the state $k\in [n]$
    To obtain the overall expression we combine both results, i.e., 
   \begin{align*}
        \left[\sum_{k\in [n]} C_t^{s}\cdot l_{t}^k + l ,\sum_{k\in [n]} C_t^{s}\cdot u_{t}^k +u\right]
    \end{align*}
    where $l$ and $u$ are the lower and upper bounds computed by $\Mon_{\Wvar_t}^{(\chi,\hori)}$ the currently selected monitor.

    \hfill $\qed$
\end{proof}

\subsection{Hidden Markov chain monitoring (proof of Theorem~\ref{thrm:hmm})}
\label{appx:hmm}
We use a Markov chain construction similar to the one in 
Section~\ref{appx:mc}.
We define the state space $\Sdomain= \Wdomain$, the initial distribution for every $\Wvar_1\in \Wdomain$ as
\begin{align*}
    \lambda(\Wvar_1)= \env(\epsilon)(\Pvar_1)\cdot \eta(\Pvar_1, \Xvar_1),
\end{align*}
and the Markov transition kernel for every $\Wvar_a, \Wvar_b \in \Wdomain$ as 
\begin{align*}
    M(\Wvar_a, \Wvar_b) = \env(\Wvar_a)(\Pvar_a)\cdot \eta(\Pvar_b, \Xvar_b).
\end{align*}
This Markov chain behaves identically to the stochastic process $\Wrandomvar$.
The Markov chain is hidden, because only $\Xvar$'s are observed, i.e., we define a label function $g\colon \Wdomain\to \{\tail, \head\}$ such that for every $(\Pvar, \Xvar)\in \Wdomain$ we have $g((\Pvar, \Xvar))= \Xvar$. 

\begin{theorem*}[\ref{thrm:hmm}]
    The monitor described above solves Problem~\ref{prob:monitoring} under Assumption~\ref{assumption:hmm}.
\end{theorem*}
\begin{proof}
    By the ergodic theorem for Markov chains (see Theorem 1.10.2~\cite{Norris_1997}) we know that for the stochastic process $\Wrandomvar$ we have 
    \begin{align*}
        \lim_{t\to\infty} \frac{1}{t}\sum_{i=1}^t \Xrandomvar_i = \sum_{(\Pvar, \Xvar)\in \Wdomain} \pi(\Pvar, \Xvar)\indi[\Xvar=\head] = \expe_{\pi}(\Xrandomvar) \quad \text{a.s}. 
    \end{align*}
    Moreover we can observe that
    \begin{align*}
       \expe_{\pi}(\Xrandomvar) &= \sum_{(\Pvar, \Xvar)\in \Wdomain} \pi(\Pvar, \Xvar)\indi[\Xvar=\head]
       = \sum_{\Pvar\in \Pdomain}  \pi(\Pvar,1)  = \sum_{\Pvar\in \Pdomain}  \prob_{\pi}(\Pvar,1) \\
       &= \sum_{\Pvar\in \Pdomain}  \prob_{\pi}(1\mid\Pvar) \prob_{\pi}(\Pvar)  
       =\sum_{\Pvar\in \Pdomain}  \Pvar( \pi(\Pvar, 0)  + \pi(\Pvar, 1))\\
       &= \expe_{\pi}(\Prandomvar)
    \end{align*}
    The monitors constructed in Henzinger et al.~\cite{henzinger2023partial} are designed for expressions such as $\expe_{\pi}(\Xrandomvar)$. The intervals computed by their monitors are sound. 
    \hfill $\qed$
\end{proof}

\subsection{Additive Dynamic Monitoring (proof of Theorem~\ref{thrm:dynamic})}
\begin{theorem*}[\ref{thrm:dynamic}]
    The monitor described in Section~\ref{subsec:facct} solves Problem~\ref{prob:monitoring} under Assumption~\ref{assumption:linear}.
\end{theorem*}
\begin{proof}
The monitor developed in Henzinger et al.~\cite{henzinger2023dynamic} is trivially sound for current fairness and predictive horizon $0$. 
Because it estimate the bias of the initial coin $\Pvar_1$, we can easily modify their monitor to capture the bias fairness measure as well. 
First, notice that the estimated value $\hat{\Pvar}_1$ and the observed outcomes can be used to estimate the bias at each point in time $t\in \NN$ by resolving the recursion, i.e., 
\begin{align*}
    \hat{\Pvar}_t = \hat{\Pvar}_1 +\sum_{i=1}^{t-1} \beta(\Xvar_i)
\end{align*}
Therefore, we can estimate bias fairness using the estimator
\begin{align*}
  \hat{\Pvar}_1 + \frac{1}{t}   \sum_{i=1}^t \sum_{j=1}^{i-1} \beta(\Xvar_i)
\end{align*}
Because we limit ourselves to additive dynamics, the error bounds remain the same. We can compute the additive shift recursively, i.e., we observe that 
\begin{align*}
     \sum_{i=1}^t \sum_{j=1}^{i-1} \beta(\Xvar_i)\sum_{j=1}^{t-1} (t-i)\cdot \beta(\Xvar_i) = \sum_{j=1}^{t} (t-i)\cdot \beta(\Xvar_i),
\end{align*}
thus we notice that 
\begin{align*}
    \sum_{j=1}^{t} (t-i)\cdot \beta(\Xvar_i) + \sum_{i=1}^{t} \beta(\Xvar_i) = \sum_{j=1}^{t} ((t+1)-i)\cdot \beta(\Xvar_i) = \sum_{j=1}^{t+1} ((t+1)-i)\cdot \beta(\Xvar_i).
\end{align*}
Hence, we simply maintain an additional register, which we increment using the value of $\regC_t$ at every time step.
\hfill $\qed$
\end{proof}

\subsection{Finite Time Window Enforcer with Probabilistic Guarantees (Proof of Theorem~\ref{thm:greedy_enforcers})}

For this enforcement problem, each bias-outcome sequence $w_{1:t}$, the only relevant information is the number of already tossed coins ($t$), and the number of coin tosses that have resulted in heads, which we denote as $h = \sum_{i=1}^t x_i$.
To decide whether to intervene in a coin toss or not, the enforcer following Eq.~\eqref{eq:best-effort-enforcer-general-delta} we computes the probability of $\env_p$ producing an outcome inside the required target interval at time $T$, conditioned on having seen a trace with $t$ tosses and $h$ heads.
This probability is denoted as 
$\prob_{\env_p}\big( \fair_O(\Wrandomvar_{1:T}) \in I_T \mid \Wvar_{1:t}\big)$ in our general framework, 
but for the purposes of this proof we will denote it as $\prob(t,h)$ to ease notation.
Following Eq.~\eqref{eq:best-effort-enforcer-general-delta}, the probabilistic enforcer with confidence parameter $\conf$ is defined, for a sequence $w_{1:t}$ with $t$ coins and $h$ tails as:
\begin{equation}
\label{eq:best-effort-enforcer-general-delta-proof}
    \Enf_\conf(w_{1:t}) = 
    \begin{cases}
        (\Pvar_t, \Xvar_t) & \mbox{ if  } \: \prob(t,h) \geq 1-\conf, \\
        \left(\Pvar_t, \arg\max_{\eps\in\{0,1\}} \prob(t,h+\eps)\right) & \mbox{ otherwise.}
    \end{cases}
\end{equation}
We call it the $\conf$-enforcer.
Let $\prob_\conf(t,h) = \prob_{\env_p, \Enf_\conf}\big( \fair_O(\Wrandomvar_{1:T}) \in I_T \mid \Wvar_{1:t}\big)$, i.e., the probability of the enforced process of producing an outcome inside the target interval.
We want to show that $\prob_\conf(0,0) \geq 1-\conf$.
We first show that the enforcer does not decrease the probability to obtain a fair outcome.
\begin{lemma}
    For all $t\in [0,T]$ and $h\in [0,t]$, it holds that
    $\prob_\conf(t,h) \geq \prob(t,h)$.
\end{lemma}
\begin{proof}
   We prove this result by induction descending on $t$.
   \paragraph{Base case.} If $t=T$, the probability to reach the target interval is either 0 or 1 independently of the enforcement, so the inequality is trivially satisfied as an equality.
   \paragraph{Inductive case.} We consider two cases, depending on whether $\prob(t,h)\geq 1-\conf$.
   \begin{itemize}
       \item If $\prob(t,h)\geq 1-\delta$, the enforcer does not intervene, so
       \[
       \prob(t,h)_\delta = p\cdot \prob_\conf(t+1,h+1) + (1-p)\cdot \prob_\conf(t+1,h).
       \]
       By the induction hypothesis, both $\prob_\conf(t+1,h+1)$ and $\prob_\conf(t+1,h)$ are greater than their unenforced counterparts, so
       \[
       \prob(t,h)_\delta \geq p\cdot \prob(t+1,h+1) + (1-p)\cdot \prob_(t+1,h) = \prob(t,h).
       \]
       \item If $\prob(t,h) < 1-\conf$, then the enforcer intervenes deterministically at this sequence, so we have 
       \begin{equation}
       \label{eq:aux7}
           \prob_\conf(t,h) = \prob_\conf(t+1,h+\eps),
       \end{equation}
       where $\eps\in\{0,1\}$.
       By the induction hypothesis, we know that
       \begin{equation}
       \label{eq:aux8}
       \prob_\conf(t+1, h+\eps) \geq \prob(t+1, h+\eps). 
       \end{equation}
       Since $\eps\in\{0,1\}$ is chosen to maximize $\prob(t+1,h+\eps)$, it is also the case that 
       \begin{equation}
           \label{eq:aux9}
            \prob_\conf(t+1, h+\eps) \geq \prob(t+1, h+(1-\eps)).   
       \end{equation}
       Putting Equations~\eqref{eq:aux7},~\eqref{eq:aux8} and~\eqref{eq:aux9} together, we have
       \begin{equation*}
            \prob_\conf(t,h) \geq p_\eps\cdot \prob(t+1, h+\eps) + (1-p_\eps)\cdot \prob(t+1, h+(1-\eps) = \prob(t,h),
       \end{equation*}
       where $p_\eps = p$ if $\eps=1$ and $p_\eps = 1-p$ otherwise.
   \end{itemize}
   \hfill $\qed$
\end{proof}

\begin{theorem*}[\ref{thm:greedy_enforcers}]
    The enforcer described in Equation~\eqref{eq:best-effort-enforcer-general-delta} solves Problem~\ref{prob:enforcement} under Assumptions~\ref{assumption:enf:finite} and \ref{assumption:stationary_known_det}.
\end{theorem*}
\begin{proof}
    We will prove the following stronger result: for all $t\in [0,T]$ and $h\in [0,t]$, 
    if $\prob(t,h)>0$, then it holds that
    $\prob_\conf(t,h)\geq 1-\conf$.
    This result directly implies that whenever the enforcement problem is feasible (i.e., $\prob(t,h)>0$), then the presented enforcer guarantees the given confidence value. 
    
    We prove this result also by induction on $t$.
    \paragraph{Base case.} If $t=T$, $\prob(T,h)$ is either $0$ or $1$, so if it is not zero, it is greater or equal than $1-\conf$ for any $\conf\in [0,1]$.
    \paragraph{Inductive case.}
    We consider two cases, depending on whether $\prob(t,h)\geq 1-\conf$.
    \begin{itemize}
        \item If $\prob(t,h)\geq 1-\conf$, the enforcer does not intervene, so by the previous lemma we have $\prob(t,h)_\conf \geq \prob(t,h)\geq 1-\conf$.
        \item If $\prob(t,h) < 1-\delta$, the enforcer intervenes, so $\prob_\conf(t,h) = \prob_\conf(t+1,h+\eps)$ for $\eps\in\{0,1\}$.
        If $\prob(t,h) >0$, then also $\prob_\conf(t+1,h+\eps) > 0$, and by the induction hypothesis $\prob_\conf(t+1,h+\eps) \geq 1-\delta$.
    \end{itemize}
    \hfill $\qed$
\end{proof}

\end{document}